\documentclass[11pt]{article}
\usepackage{times}
\usepackage[margin=1in]{geometry}
\usepackage[utf8]{inputenc}
\usepackage[T1]{fontenc}
\usepackage{microtype}
\usepackage{amsmath,amssymb,mathtools,amsthm}
\usepackage{booktabs,tabularx,array,multirow}
\usepackage{enumitem,graphicx,xcolor,xurl,url,hyperref,natbib,cleveref}
\hypersetup{hidelinks}
\newtheorem{proposition}{Proposition}

\title{Authority Before Utility: Non-Compensatory Control for Persistent LLM Memory}
\author{Wesley Shu\\The Institute of Energetic Paradigm}
\date{September 2026}

\begin{document}
\maketitle

\begin{abstract}
Persistent memory creates a control problem that retrieval relevance alone does not solve: a memory can remain highly useful after an update, deletion, or revocation makes it inadmissible for the current answer. We formalize this as a separation between \emph{utility} and \emph{authority}. A fixed finite penalty applied to an unnormalized utility score cannot guarantee exclusion under arbitrary positive-affine reparameterization of that score; by contrast, rank-normalized compensation is scale-invariant and therefore forms a stronger empirical comparator. Our prospectively frozen TIDE/LongMemEval primary was quarantined before a valid HELDOUT comparison because the materialized TIDE adapter conflated historical age with query-relative inadmissibility and the aligned LongMemEval split left no DEV set for the predeclared penalty selection. We therefore report a post-primary replacement diagnostic on Memora Remembering, where update/delete operations provide item-level forgetting state. On Qwen3-8B, DEV selected $\lambda=0.6$ from a ten-point normalized SOFT family. Across 185 HELDOUT units in 28 dependency clusters, HARD exclusion yields 4.04\% balanced construct error versus 19.66\% for locked SOFT, a paired difference of 15.61 points with a 20,000-replicate cluster-bootstrap 95\% interval of [13.07,18.76]. The effect is driven primarily by forgotten-value leakage while current-value recall is preserved. This is same-$Q$ operator-comparison evidence, not a universal claim that scalar control fails, not an evaluation of learned authority inference, and not an independent downstream-harm endpoint.
\end{abstract}

\section{Introduction}
Long-lived language-model agents increasingly externalize memory: past exchanges are stored, retrieved, summarized, ranked, and reinserted into later prompts. This architecture appears in retrieval-augmented generation \citep{lewis2020rag}, nearest-neighbor language models \citep{khandelwal2020knnlm}, retrieval-enhanced transformers \citep{borgeaud2022retro}, memorizing transformers \citep{wu2022memorizing}, explicit long-term memory modules \citep{wang2023longmem}, and interactive agents with persistent experience stores \citep{park2023generative,zhong2024memorybank}. Long-context models do not eliminate the control problem: even when information can be placed in context, models use long contexts unevenly and are sensitive to position and distraction \citep{liu2024lost}.

The usual objective in these systems is to recover information that is relevant, similar, recent, useful, or predictive. That objective is incomplete when memory changes over time. A deleted task, superseded preference, corrected address, revoked instruction, or obsolete value may remain semantically close to the query and highly predictive of the user's history while no longer being permitted to determine the current answer. Long-term-memory benchmarks increasingly expose this update problem \citep{maharana2024locomo,wu2025longmemeval,uddin2026memora,patel2026supersede,tao2026memconflict}. The central distinction in this paper is therefore between \emph{utility}: how useful a memory would be if used, and \emph{authority}: whether it is currently admissible to use.

This separation has deep precedents outside LLM memory. Security architectures distinguish possession of information from permission to use it, from classic protection and information-flow models \citep{saltzer1975protection,denning1976lattice} through role-based and usage control \citep{sandhu1996rbac,ferraiolo2001nist,park2004ucon,zhang2005ucon}. In particular, usage control emphasizes decision continuity, mutability, and immediate revocation rather than a one-time relevance judgment. Decision theory likewise distinguishes compensatory aggregation from lexicographic or otherwise non-compensatory preference structures \citep{fishburn1974lexicographic,fishburn1976noncomp}. Measurement theory warns that conclusions can depend on which transformations of a numerical scale are meaningful \citep{stevens1946scales,narens1981scales}. Constrained optimization and safe control provide a complementary lesson: constraints and objectives need not be represented by the same mechanism \citep{han1979exact,dipillo1989exact,garcia2015safe,achiam2017cpo,alshiekh2018shielding,brunke2022safe}.

Our question is deliberately narrower than ``hard gating is better.'' We ask: \emph{holding the authority state $Q$ fixed, does non-compensatory exclusion change forgetting-sensitive answer reliability relative to a strong finite compensatory controller?} The qualification matters. Exact penalties can enforce constraints under problem-dependent conditions, and a scalar controller with an explicit threshold can reproduce hard feasibility. We therefore compare HARD exclusion to a scale-normalized finite SOFT family rather than to an obviously fragile raw-score penalty.

The execution history is itself part of the scientific result. We prospectively froze a TIDE \citep{sobhani2026tide}/LongMemEval \citep{wu2025longmemeval} matrix over Qwen3, Llama-3.1, and Gemma-3 \citep{yang2025qwen3,grattafiori2024llama3,gemmateam2025gemma3}. Before a defensible HELDOUT comparison, construct audits showed that our TIDE adapter confused historical age with query-relative inadmissibility, while the aligned LongMemEval split left no DEV set for the predeclared $\lambda$ selection. We quarantined those cells rather than convert a protocol failure into a nominal positive result.

We then evaluated a replacement Memora Remembering cell. Memora explicitly includes updates/deletions and scores forgetting-aware behavior \citep{uddin2026memora}. The replacement result is large, but its evidence class is intentionally limited because the utility proxy and deterministic scorer were finalized during replacement development. The paper therefore makes three contributions: (1) a scale-representation argument explaining why fixed raw finite penalties do not define representation-independent exclusion; (2) a cross-literature formulation of persistent-memory authority as an admission problem distinct from retrieval utility; and (3) an auditable replacement experiment showing a substantial same-$Q$ difference between HARD exclusion and the DEV-selected member of a normalized finite-SOFT family.

\section{Authority Before Utility}
Let a query expose candidate memory items $i\in\mathcal M$. Each item has a raw utility score $U_i\in\mathbb R$ produced by some retrieval or ranking mechanism and an authority state $m_Q(i)\in\{0,1\}$ supplied by a frozen object $Q$, where $m_Q(i)=1$ means that item $i$ is inadmissible for the present answer. We condition on $Q$ throughout. Learning or inferring $Q$ from raw dialogue is a separate problem.

\paragraph{Hard admission.} HARD defines the feasible set
\[
\mathcal F(Q)=\{i\in\mathcal M:m_Q(i)=0\}
\]
and ranks only within $\mathcal F(Q)$. This makes admissibility logically prior to utility ranking.

\paragraph{Finite compensation.} A common alternative is to retain all candidates and subtract a finite penalty,
\[
S_\lambda(i)=U_i-\lambda m_Q(i),\qquad \lambda<\infty.
\]
An inadmissible item can then survive whenever its utility advantage exceeds the penalty. The issue is not merely whether one particular $\lambda$ is numerically large enough; it is whether the exclusion claim survives admissible reparameterizations of the utility representation.

\begin{proposition}[Raw finite-penalty rescaling]\label{prop:finite-penalty-rescaling}
Fix any finite $\lambda$. Suppose an inadmissible candidate $j$ has positive raw-utility advantage $U_j-U_k>0$ over some admissible candidate $k$. Then there exists a positive-affine transformation $U'_i=aU_i+b$, $a>0$, under which $j$ outranks $k$ according to $U'_i-\lambda m_Q(i)$.
\end{proposition}
\begin{proof}
The transformed score difference is $a(U_j-U_k)-\lambda$. Choosing $a>\lambda/(U_j-U_k)$ makes this quantity positive; the additive constant $b$ cancels. Hence a fixed raw penalty cannot guarantee exclusion across positive-affine representations whenever an inadmissible item has a positive utility advantage.
\end{proof}

The proposition is a representation statement, not an impossibility theorem for scalar control. Measurement theory distinguishes scales by transformations under which empirical statements remain meaningful \citep{stevens1946scales,narens1981scales}. If only ordering matters, any increasing transformation preserves ranks, whereas a fixed additive penalty has units tied to the chosen numerical representation. This is why the paper does not treat an arbitrary raw-score penalty as a strong baseline.

\paragraph{Scale-normalized compensation.} Let $R_i\in[0,1]$ be a rank-normalized utility proxy. We test
\[
S^{\rm norm}_\lambda(i)=R_i-\lambda m_Q(i).
\]
Because $R_i$ depends only on order, positive-affine transformations of raw $U$ do not change the normalized ranking. The finite-SOFT family is therefore invariant to the specific positive-affine parameterization attacked by Proposition~\ref{prop:finite-penalty-rescaling}. HARD remains distinct because infeasible items are removed rather than merely displaced in a shared score.

\paragraph{Relation to exact penalties.} Exact-penalty theory establishes that finite penalties can recover constrained optima under regularity conditions and problem-dependent parameter bounds \citep{han1979exact,dipillo1989exact}. Our claim is compatible with that literature. A thresholded scalar scheme can also reproduce the same feasible set as HARD. The empirical question here is only HARD versus the explicitly tested normalized finite family under common $Q$, candidates, model, prompt, and scorer.

\section{Related Work}
\subsection{External memory, long context, and persistent agents}
Retrieval-augmented language modeling makes external state directly useful at inference time. $k$NN-LM interpolates a neural language model with nearest-neighbor retrieval \citep{khandelwal2020knnlm}; RAG combines parametric generation with a non-parametric document index \citep{lewis2020rag}; RETRO retrieves from a very large text database during generation \citep{borgeaud2022retro}; and Memorizing Transformers add approximate nearest-neighbor access to past internal representations \citep{wu2022memorizing}. LongMem extends this line with a decoupled long-term memory network \citep{wang2023longmem}. These systems establish the value of externally accessible state, but their main objective is to find useful information, not to represent whether useful information remains authorized after mutation.

Persistent agents make the temporal problem unavoidable. Generative Agents store experiences, synthesize reflections, and retrieve memories to guide later behavior \citep{park2023generative}. MemoryBank explicitly adds long-term user memory and time-sensitive memory updating \citep{zhong2024memorybank}. LoCoMo evaluates conversations spanning many sessions and long-range temporal/causal dependencies \citep{maharana2024locomo}; LongMemEval isolates extraction, multi-session reasoning, temporal reasoning, knowledge updates, and abstention \citep{wu2025longmemeval}. Memora goes further by explicitly evaluating remembering and forgetting under updates/deletions \citep{uddin2026memora}. Recent diagnostics such as Supersede and MemConflict target supersession and conflict handling \citep{patel2026supersede,tao2026memconflict}, while authority-oriented retrieval work asks which source remains controlling rather than merely similar \citep{bacellar2026car}. This progression motivates our focus on current admissibility rather than memory capacity alone.

Long context is not a substitute for admission control. \citet{liu2024lost} show that models can use the same relevant evidence differently depending on its position in a long context. In our setting, this makes context occupancy a potential mechanism: excluding inadmissible items changes not only their direct availability but also which admissible items occupy limited prompt space. We therefore treat context reshaping as an unresolved alternative explanation rather than attributing the full observed effect to a pure logical-gating mechanism.

\subsection{Forgetting, updating, and deletion}
The continual-learning literature studies a different but adjacent failure: new learning can destroy performance on old tasks. Elastic weight consolidation made catastrophic forgetting a prominent modern benchmark problem \citep{kirkpatrick2017ewc}, and major reviews organize the resulting methods and scenarios \citep{parisi2019continual,delange2022continual,vandeven2022types}. Our problem is almost the inverse: some old information should remain retrievable, while other old information should cease to control the answer after a state change. Thus ``remember everything'' and ``forget obsolete state'' are separate requirements.

Machine unlearning provides a second neighboring literature. Early work framed system forgetting as removal of data and its lineage \citep{cao2015unlearning}; later surveys distinguish deletion from removal of a training sample's influence on a learned model \citep{xu2024unlearning}. Persistent external memory is operationally simpler in one respect---an item can be withheld without retraining model weights---but harder in another: deleted or superseded content may still be present in stores, summaries, replicas, or retrieved context. Our experiment does not claim certified unlearning. It tests answer-time admission under supplied item-level authority labels.

\subsection{Authorization, non-compensation, and runtime constraints}
Computer-security research has long treated access as a permission relation rather than a relevance score. Classic protection design separates authorization mechanisms from application utility \citep{saltzer1975protection}; lattice information-flow models impose structural constraints on permissible flows \citep{denning1976lattice}. RBAC organizes permissions through roles \citep{sandhu1996rbac,ferraiolo2001nist}. Most directly relevant, UCON extends access control with mutable attributes, obligations, conditions, decision continuity, and ongoing revocation \citep{park2004ucon,zhang2005ucon}. Persistent LLM memory differs in object and semantics, but the architectural lesson is close: a long-lived system needs a control state that can change after information was once legitimate.

Non-compensatory choice has an established decision-theoretic lineage. Lexicographic rules give priority to one criterion such that gains elsewhere cannot compensate for failure on the higher-priority criterion \citep{fishburn1974lexicographic}; Fishburn separately formalized noncompensatory preferences \citep{fishburn1976noncomp}. Safe reinforcement learning likewise distinguishes reward optimization from safety constraints \citep{garcia2015safe,achiam2017cpo}; shielding enforces specifications by filtering actions around a learned policy \citep{alshiekh2018shielding}; and modern safe-learning surveys treat certified or runtime constraint enforcement as a distinct design axis \citep{brunke2022safe}. Our contribution is not the generic discovery that constraints can dominate objectives. It is the scale-representation argument plus a same-$Q$ persistent-memory operator comparison under forgetting-aware evaluation.

\section{Execution Audit and Replacement Design}
\subsection{Why the frozen primary was quarantined}
The prospectively frozen primary crossed TIDE and LongMemEval with Qwen3, Llama-3.1, and Gemma-3, plus a frontier-model stress test. The matrix was intended to test the same control distinction across temporal-version and conversational-memory settings. It did not survive construct/procedure audit.

TIDE is a version-resolution benchmark in which earlier official text can remain correct for a historically scoped query \citep{sobhani2026tide}. Under our materialized date adapter, answer-bearing historical evidence was labeled inadmissible in units from both DEV and HELDOUT. The adapter therefore partially encoded ``old'' as ``unauthorized,'' which is not the construct we intend. For aligned LongMemEval, the frozen dependency split placed all aligned units in HELDOUT, leaving no DEV examples for the predeclared per-cell $\lambda$ selection. Running those cells anyway would have mixed a construct error with a protocol deviation. We therefore report neither benchmark as a completed primary cell.

This quarantine matters for interpretation. The replacement analysis below is not presented as if it were the originally preregistered six-cell experiment. It is a post-primary diagnostic developed after the primary failed its own validity checks.

\subsection{Memora materialization and authority state}
We pinned the public Memora repository to the commit identified in the evidence object (short form \texttt{a6493188efc8...87da30ff}; Appendix~\ref{app:evidence}). Memora's Remembering task includes current information and explicit forgetting evidence created by update/delete operations \citep{uddin2026memora}. We materialize explicit current evidence as $q=0$ and forgotten evidence as $q=1$. The adapter yields 200 units, 998 current candidates, and 1,476 forgotten candidates.

Dependency clusters are defined as period$\times$persona before outcome scoring. Hashing these clusters produces 15 DEV and 185 HELDOUT units with no cluster leakage. This split is fixed for the replacement experiment.

The source-accounting edge case is deterministic. Of 200 Remembering questions, 196 have nonempty source \texttt{forgotten\_items}. Exactly 180 contain value-bearing forgotten items that the evidence-value adapter can materialize; 16 contain session identifiers without a forgotten value string. Four further questions have empty forgotten-item lists. Because the scorer tests textual presence/absence of values, the 16 session-only cases cannot contribute a value-level forgetting target. The evidence object enumerates all 196 IDs and their classification. A post-hoc sensitivity restricted to HELDOUT units with materialized forgotten values (167 units) yields a 17.30-point paired separation: HARD 1.78\% versus SOFT 19.08\% error.

\subsection{Controllers and locked model}
HARD removes every $q=1$ candidate before ranking. SOFT retains all candidates and ranks
\[
R_i-\lambda q_i,\qquad
\lambda\in\{0,.05,.10,.20,.40,.60,.80,.90,.95,.99\},
\]
where $R_i\in[0,1]$ is a rank-normalized session-recency proxy. At most 12 candidates enter the prompt. The executed checkpoint is Qwen3-8B \citep{yang2025qwen3}, run deterministically with thinking disabled and a 256-token output ceiling.

DEV contains 15 units. The complete SOFT balanced-error surface is
\[
.2783,.2650,.2550,.2325,.2134,\mathbf{.1912},.1981,.2118,.2118,.2118,
\]
so the locked selection rule chooses $\lambda=.6$; ties favor the smaller $\lambda$. HARD DEV error is .0733. The small DEV set and ten-way selection are important limitations rather than incidental details.

\subsection{Measurement and estimand}
The deterministic scorer extracts target values that should be current and target values that should be forgotten. Current recall is the fraction of current target values appearing in the answer. Forgetting absence is the fraction of forgotten target values absent from the answer. We report balanced construct error
\[
E=\tfrac12[(1-\mathrm{current\ recall})+(1-\mathrm{forgetting\ absence})].
\]
This endpoint deliberately couples answer content to the forgetting construct. It is not an independent measure of downstream harm, user satisfaction, factuality in general, or policy compliance outside the supplied $Q$.

The causal object is correspondingly narrow. HARD and SOFT share $Q$, candidate pool, checkpoint, prompt template, generation settings, and scorer, and differ in how $Q$ enters admission/ranking. But HARD can change prompt occupancy by deleting candidates, so the locked experiment does not isolate a pure ``logical non-compensation'' mechanism from benefits of freeing context slots or changing distractor composition.

The replacement pipeline is also not fully prospective. The session-recency utility proxy was materialized during replacement execution, and the deterministic scorer was finalized after DEV outputs existed. HELDOUT was not used to design either component, but this DEV adaptivity is sufficient to classify the result as operator-comparison evidence rather than prospective causal confirmation.

\section{Results}
Table~\ref{tab:mainresult} reports the locked HELDOUT comparison. Across 185 units in 28 dependency clusters, HARD balanced construct error is 4.04\% and SOFT($\lambda=.6$) error is 19.66\%, giving a paired SOFT-minus-HARD difference of 15.61 percentage points. A 20,000-replicate percentile bootstrap over the 28 clusters gives a 95\% interval of [13.07,18.76] points.

\begin{table}[t]
\centering\small
\caption{Locked Qwen3-8B HELDOUT replacement cell (185 units; 28 period$\times$persona dependency clusters).}
\label{tab:mainresult}
\begin{tabular}{lrr}
\toprule
Metric & HARD & SOFT ($\lambda=.6$)\\
\midrule
Balanced construct error $\downarrow$ & 4.04\% & 19.66\%\\
Current recall $\uparrow$ & 91.92\% & 89.52\%\\
Forgetting absence $\uparrow$ & 100.00\% & 71.17\%\\
\bottomrule
\end{tabular}
\end{table}

The decomposition is important. Current recall differs by 2.40 points, whereas forgetting absence differs by 28.83 points. Thus the aggregate gap is primarily a forgetting-sensitive answer/admission effect rather than a large gain in recalling current values. At the unit level HARD is better on 95 units and tied on 90; SOFT is better on none under this scorer.

\paragraph{Cluster robustness.} Cluster-mean paired effects range from 6.67 to 40.00 points. Leave-one-cluster-out aggregate effects range from 14.94 to 16.13 points, so no single period$\times$persona cluster explains the aggregate result. Because personas recur across periods, we also performed post-primary re-clustering sensitivities. A 20,000-replicate persona-only bootstrap (10 clusters) gives an approximate 95\% interval of [13.80,18.01] points; a period-only bootstrap (3 clusters) gives approximately [10.76,22.66]. Both preserve the sign but are descriptive robustness checks, not a redefinition of the primary inferential unit.

\section{Interpretation}
The result is most naturally read as evidence about control architecture under known state. Once an item is labeled inadmissible, a controller can either remove it from the candidate set or continue to place it on a common utility scale and attempt to suppress it. The latter remains a legitimate design when the scale and admission threshold are explicitly engineered; exact-penalty theory makes that clear. What the present experiment shows is narrower: under the tested normalized compensatory family, the best DEV-selected finite penalty still admits substantially more forgotten-value evidence into answers than hard exclusion, without buying higher current-value recall.

This distinction helps connect persistent memory to older control literatures. In access control, permission is not generally treated as another feature whose weight can be compensated by task utility \citep{saltzer1975protection,park2004ucon}. In safe control, a runtime shield or constrained optimizer can keep feasibility separate from reward maximization \citep{achiam2017cpo,alshiekh2018shielding,brunke2022safe}. And in noncompensatory decision rules, a higher-priority condition may block a choice regardless of improvements on lower-priority attributes \citep{fishburn1974lexicographic,fishburn1976noncomp}. Persistent LLM memory creates an analogous systems question: should evidence that was once legitimate remain eligible merely because it is useful now?

The answer is not automatically ``always hard-delete.'' Authority can be uncertain, contested, hierarchical, or time-scoped. A system may need provenance, appeal, expiry, revocation, and conflict-resolution mechanisms rather than a single Boolean bit. Our known-$Q$ experiment intentionally avoids that upstream problem. Its role is to isolate what happens after an authority state has already been supplied.

\section{Limitations and Claim Boundary}
The central quantitative claim is local. It applies to one Qwen3-8B Memora Remembering replacement cell, the supplied item-level $Q$, a 12-candidate context budget, a session-recency ranking proxy, and a deterministic evidence-value scorer. It does not establish generalization across models, benchmarks, learned authority estimators, prompt budgets, retrieval systems, or alternative scalar/admission controllers.

The DEV set is small: 15 units select among ten $\lambda$ values. The HELDOUT bootstrap conditions on the selected $\lambda$ and therefore does not propagate DEV-selection uncertainty. A repeated-split or nested-selection design would be stronger.

The endpoint is construct-coupled. When HARD has oracle-known $q=1$ labels, 100\% forgotten-value absence is partly expected because those values are removed from the prompt. The informative accompanying observation is that current recall is not reduced in this cell. Even so, balanced construct error is not an independent downstream-harm measure, and the experiment does not quantify user-level consequences of a forgotten value reappearing.

HARD also changes context composition. Candidate deletion can free prompt slots and reduce distraction. Because long-context utilization is position- and context-sensitive \citep{liu2024lost}, the experiment cannot identify how much of the effect arises from logical non-compensation versus context reshaping. A stronger factorial design would hold context occupancy fixed by replacing excluded items with matched admissible distractors or explicit null slots.

The experiment assumes correct authority labels. Inferring whether a memory is actually superseded, deleted, revoked, jurisdictionally controlling, or merely historically scoped is a separate learning and governance problem. Our failed TIDE adapter demonstrates why this distinction matters: historical age is not itself inadmissibility. Future work should evaluate learned or uncertain $Q$ against independently adjudicated authority state rather than folding authority inference into the controller and calling the result a gating test.

Finally, the replacement result is not the prospectively frozen primary. The original TIDE/LongMemEval matrix failed construct/split checks and remains quarantined. The Memora utility proxy and scorer were finalized during replacement development. This lowers the evidentiary status of the replacement cell even though HELDOUT was preserved during that development.

\section{Conclusion}
Persistent memory creates a control problem that memory capacity and retrieval quality alone cannot solve. Information can remain relevant after it ceases to be admissible. Separating authority from utility makes that failure mode explicit and connects LLM memory to mature ideas in measurement theory, noncompensatory decision rules, access control, exact penalties, and safe runtime constraint enforcement.

A fixed finite penalty on an unnormalized utility representation cannot guarantee exclusion across positive-affine rescalings; rank normalization removes that specific representation defect and therefore supplies a stronger compensatory baseline. In our post-primary Memora/Qwen3-8B replacement cell, HARD exclusion substantially reduces forgetting-sensitive answer/admission error relative to the DEV-selected normalized finite-SOFT controller while preserving current-value recall. The result supports further study of authority-aware memory control, but it does not establish that every scalar controller is inferior, that authority can be inferred reliably, or that the measured gap is itself a downstream-harm estimate.

\section*{AI Use Statement}
Generative AI tools were used in this work to help develop conceptual frameworks; formulate and refine mathematical claims; assist with proof development; propose or refine hypotheses; design and critique research methodology and experiments; implement methods; clean or reformat data and benchmark artifacts; interpret experimental results; draft and edit manuscript text; search, summarize, and organize literature; identify research gaps; improve readability; format references; and suggest manuscript structure. They were not used to generate synthetic research datasets, transcribe research recordings, formulate survey/interview questions, or generate scientific figures used as evidence.

All AI-assisted research outputs were subjected to artifact-level or author review as appropriate. Executed benchmark results were checked against materialized call plans, row-level outputs, deterministic scorers, hashes, and independent aggregate recomputation. Literature and citation claims were checked against source records. The authors take responsibility for the final manuscript, claims, code, analyses, citations, and artifacts.

\appendix
\section{Replacement Evidence Object}\label{app:evidence}
The replacement evidence object is included with the submission supplementary material as \nolinkurl{ICLR2_MEMORA_QWEN_EVIDENCE_OBJECT_20260917.zip}, SHA-256 \nolinkurl{5f04ba5e682ea4ab83e2d172dbb9c3f0dedd8ce453a716e3978b596c3aeb3981}. It contains the exact Memora commit binding, adapter and plan builders, split/cluster map, item-level materialized units and $Q$ labels, complete DEV curve, scorer implementations, call plans, raw local-execution rows, row-level scores, HELDOUT aggregate, 196-vs-180 reconciliation, materialized-forgetting sensitivity, cluster robustness, and a per-file SHA-256 manifest. The executed model checkpoint is Qwen3-8B; the original Llama/Gemma planned cells were not executed and are not claimed.

\section{Exact Model and Runtime Binding}
The executed local checkpoint is Qwen3-8B, Hugging Face revision \nolinkurl{b968826d9c46dd6066d109eabc6255188de91218}. The audited runtime used Python 3.12.3, PyTorch 2.14.0+cpu, Transformers 4.57.6, FastAPI 0.141.1, Uvicorn 0.53.0, and OpenAI Python 3.14.1. Generation used BF16 model weights, \texttt{do\_sample=False}, thinking disabled, and a 256-token output ceiling. Exact same-unit HARD/SOFT requests were deterministically deduplicated; no cross-unit response reuse was observed.

\section{DEV Surface and Accounting Reconciliation}
The ten SOFT DEV errors are reported for
\[
\lambda=(0,.05,.10,.20,.40,.60,.80,.90,.95,.99),
\]
with corresponding errors
\[
(.2783,.2650,.2550,.2325,.2134,.1912,.1981,.2118,.2118,.2118).
\]
The selected value is $.6$. HARD DEV error is $.0733$.

Among 200 Remembering questions, 196 have nonempty source \texttt{forgotten\_items}. Exactly 180 contain value-bearing forgotten items; 16 contain session IDs without value strings. The supplementary evidence object enumerates all 196 IDs and source session IDs. The executed adapter cannot invent missing target strings, so those 16 are marked \nolinkurl{SESSION_ONLY_NO_VALUE}. Four additional questions have empty forgotten-item lists.

\section{Cluster Robustness}
Across the 28 HELDOUT dependency clusters, cluster-mean paired SOFT-minus-HARD error differences range from .0667 to .4000. Leave-one-cluster-out aggregate differences range from .1494 to .1613. At unit level HARD is better on 95 units and tied on 90; SOFT is better on none under the deterministic evidence-value scorer. These diagnostics were computed after the locked primary replacement result and are descriptive robustness checks.

\end{document}